\documentclass[12pt]{article}

\usepackage{amsmath}
\usepackage{amssymb}
\usepackage{amsthm}
\usepackage{subfigure, graphicx}

\newtheorem{theorem}{Theorem}
\newtheorem{lemma}{Lemma}
\usepackage{amssymb,amsfonts,amsmath}
\usepackage{paralist}
\usepackage{graphicx}
\usepackage{subfigure}
\usepackage{paralist}
\usepackage{multirow}
\usepackage{booktabs}
\usepackage{tabularx}
\usepackage{mdwlist}
\usepackage{url}
\usepackage{verbatim}

\usepackage{eso-pic}
\usepackage{xspace}
\makeatletter
\DeclareRobustCommand\onedot{\futurelet\@let@token\@onedot}
\def\@onedot{\ifx\@let@token.\else.\null\fi\xspace}

\def\ie{\emph{i.e}\onedot}

\makeatother

\def\AAT{{$P \kern-2.0pt P \kern-1.0pt ^T$ }}
\def\ATA{{$P \kern-1.0pt ^T \kern-2.0pt P$ }}
\def\AATe{{X \kern-2.0pt X \kern-1.0pt ^T }}
\def\ATAe{{X \kern-1.0pt ^T \kern-2.0pt X }}

\def\pdd{{p_{\kern-2pt\lower+1pt\hbox{..}}\kern+1pt }}
\def\pdds{{p_{\kern-2pt\lower+1pt\hbox{..}}^{1/2}\kern+1pt}}
\def\sdd{{s_{\kern-2pt\lower+1pt\hbox{..}}\kern+1pt }}
\def\sdds{{s_{\kern-2pt\lower+1pt\hbox{..}}^{1/2}\kern+1pt}}

\def\pdc#1{{p_{\kern-2pt\lower+1pt\hbox{.}\lower+2pt\hbox{\footnotesize $#1$}\kern+1pt }}}
\def\sdc#1{{s_{\kern-2pt\lower+1pt\hbox{.}\lower+2pt\hbox{\footnotesize $#1$}\kern+1pt }}}

\def\prd#1{{p_{\kern-1pt\lower+2pt\hbox{\footnotesize $#1$}\lower+1pt\hbox{.}\kern+1pt }}}
\def\srd#1{{s_{\kern-1pt\lower+2pt\hbox{\footnotesize $#1$}\lower+1pt\hbox{.}\kern+1pt }}}

\def\xb{{\bf x}}

\def\vv1k{{{\bf v}^1 \cdots {\bf v}^k}}
\def\xxx1n{{{\bf x}_1 \cdots {\bf x}_n}}

\def\Dm1{{D^{-1}}}
\def\Dm12{{D^{-1/2}}}
\def\Dp12{{D^{1/2}}}
\def\Dxm12{{D_r^{-1/2}}}
\def\Dym12{{D_c^{-1/2}}}
\def\Dxp12{{D_r^{1/2}}}
\def\Dyp12{{D_c^{1/2}}}

\def\YYT{{Y \kern-0.0pt Y \kern-1.0pt ^T }}
\def\YTY{{Y \kern-1.0pt ^T \kern-0.0pt Y }}
\def\YT{{Y \kern-1.0pt ^T}}
\def\XT{{X \kern-1.0pt ^T}}
\def\XTX{{X \kern-1.0pt ^T \kern-2.0pt X }}

\def\FFTD {{F \kern-1.0pt F^T \kern-3.0pt D }}

\newcommand{\heading}[1]{{\bf \noindent #1}}
\newcommand{\Heading}[1]{\vspace{0.5\baselineskip}\heading{#1}\xspace}

\def\bbox{{\hfill {$\sqcap \kern-6.0pt \lower+2.4pt\hbox{--} \kern+2.7pt $}}}
\def\squa {{{$\sqcap \kern-6.0pt \lower+2.4pt\hbox{--} \kern+2.7pt $}}}

\def\knn{$k$-Nearest Neighbor\xspace}
\def\knna{$k$-NN\xspace}
\def\knnf{\knn (\knna)\xspace}

\newcolumntype{Z}{>{\centering\arraybackslash}X}

\newcommand{\kth}[1][k]{$#1$-th\xspace}

\newcommand{\sbr}[1]{\left\{#1\right\}}

\newcommand{\spr}[1]{\left(#1\right)}

\def\tblcap{Table~}
\newcommand{\tblref}[1]{\tblcap\ref{#1}}
\def\fgcap{Figure~}
\newcommand{\fgref}[1]{\fgcap\ref{#1}}

\def\mck{\mathcal{K}}

\def\mck{\mathcal{K}}

\newcommand{\xcol}[1][j]{\xb_{#1}}

\newcommand{\normv}[1]{\left\|#1\right\|}

\def\mn{Robust Linear Discriminant Analysis\xspace}
\def\mna{rLDA\xspace}
\def\mnf{\mn (\mna)\xspace}
\usepackage{algorithm}
\usepackage{algorithmic}
\usepackage{epstopdf}

\begin{document}
%
\title{Robust Linear Discriminant Analysis Using Ratio Minimization of {\LARGE   $\ell_{1,2}$}-Norms}
%
%
\author{Feiping Nie, Hua Wang, Zheng Wang, Heng Huang
\thanks{Feiping Nie and Zheng Wang are with the School of Computer Science and Center for OPTical IMagery Analysis and Learning (OPTIMAL), Northwestern Polytechnical University, Xi’an 710072, Shaanxi, P. R. China. Email: feipingnie@gmail.com, zhengwangml@gmail.com}
\thanks{Hua Wang is with the Department of Computer Science, Colorado School of Mines, USA. Email: huawangcs@gmail.com}
\thanks{Heng Huang is with the Department of Electrical and Computer Engineering, University of Pittsburgh, Pittsburgh, PA 15260 USA. Email: heng.huang@pitt.edu}
}
\maketitle

\begin{abstract}
As one of the most popular linear subspace learning methods, the Linear Discriminant Analysis (LDA) method has been widely studied in machine learning community and applied to many scientific applications. Traditional LDA minimizes the ratio of squared $\ell_2$-norms, which is sensitive to outliers.
In recent research, many $\ell_1$-norm based robust Principle Component Analysis methods were proposed to improve the robustness to outliers. However, due to the difficulty of $\ell_1$-norm ratio optimization, so far there is no existing work to utilize sparsity-inducing norms for LDA objective. In this paper, we propose a novel robust linear discriminant analysis method based on the $\ell_{1,2}$-norm ratio minimization. Minimizing the $\ell_{1,2}$-norm ratio is a much more challenging problem than the traditional methods, and there is no existing optimization algorithm to solve such non-smooth terms ratio problem.
We derive a new efficient algorithm to solve this challenging problem, and provide a theoretical analysis on the convergence of our algorithm.
The proposed algorithm is easy to implement, and converges fast in practice. Extensive experiments on both synthetic data and nine real benchmark data sets show the effectiveness of the proposed robust LDA method.
\end{abstract}


\section{Introduction}

As an important machine learning technique, Fisher linear discriminant analysis (LDA) has been successfully applied to many scientific applications in the past few years. As a subspace analysis approach to learn the low-dimensional structure of high-dimensional data, LDA seeks for a set of vectors that maximize Fisher Discriminant Criterion. LDA is designed to find a projection maximizing the class separation in a lower dimension space, \emph{i.e.} it simultaneously minimizes the within-class scatter and maximizes the between-class scatter in the projective feature vector space.

The traditional LDA uses the ratio-trace objective and has closed form solution as an eigenvalue problem. However, its solution requires the inversion of the within-class scatter matrix. Hence, a singular within-class scatter matrix results in an ill-conditioned LDA formulation. It is common to encounter rank-deficient within-class scatter matrices for high-dimensional feature spaces or for feature spaces that have highly correlated features, such as classifications for image/video, gene expression. Usually the Principal Component Analysis (PCA) is employed as pre-processing step to discard the null space of the overall scatter matrix before the LDA is used. To solve the null space problem, many variations of LDA methods have been proposed
\cite{Chu10PR,GLDAyeJMLR05,Ye08JMLR,Ye06JMLR,Ye07ICML,Choo09SDM,Kim06TKDE,Park08PR}
in machine learning communities. More recently, the trace-ratio LDA objective has been studied and shown with promising results \cite{nie2019submanifold,WangTraceRatio,traceratioTNN09}. Although the trace-ratio objective is more difficult to optimize than the ratio-trace LDA objective, it can naturally avoid the singularity problem of scatter matrix. Our work focuses on the trace-ratio objective.

It is well known that the traditional PCA and LDA use the least squares estimation objectives (based on squared $\ell_2$-norm), which are prone to the presence of outliers, because the squared large errors can dominate the sum. From a statistical point of view \cite{Robust81}, the robustness of a method is defined as the property of being insensitive to outliers. For example, in Fig.~(\ref{toy}), we plot two groups of data sampled from two different Gaussian distributions with a few of outliers. The projection directions obtained by LDA and robust LDA (will be introduced later in this paper) are calculated and visualized. Although we only have a small number of outlier data points (three outliers in each group), the projection direction of LDA (green line) has been deviated a lot from the optimal one (such as the projection direction of the proposed robust LDA -- purple line) that can separate the non-outlier data points under the projection.

In literature, several LDA methods \cite{FidlerRLDA,KimRLDA,CrouxRLDA,zhao2018new} were proposed to improve the robustness of projection subspaces by using data re-sampling and subspace searching ways. In \cite{KimRLDA}, the authors assume the class means and class covariances of data are uncertain in the binary class case, and try to find an optimal projection for the worst-case means and covariances without directly handling outliers. However, all these methods didn't replace the traditional LDA objectives by the new robust formulation, hence they didn't solve the objective function deficiency of sensitivity to outliers. It is well-known in machine learning community, to fundamentally improve the robustness of methods, the squared $\ell_2$-norm in loss functions should be changed to the correct sparsity-inducing norms. Based on this idea, many previous works have been done to improve the robustness of PCA via using the sparsity-inducing norms in the objectives \cite{Baccini96,Gao08PCA,Ke:L1PCA,R1PCA,Kwak08PAMI,Wright09NIPSRPCA,rpcaICML14}.

Although there exist some methods using the $\ell_1$-norm to improve the robustness of LDA \cite{li2015robust,ye2016l1,liu2016non,ye2016l1}, but it is very difficult to solve the optimization problems brought by the $\ell_1$-norm based objectives. Wherein, several existing $\ell_1$-norm based LDA algorithm \cite{zhong2013linear} calculate the projections one by one, which is time-consuming. Because all LDA objectives (either ratio or subtraction) have to simultaneously minimize the within-class scatter and maximize the between-class scatter, all current optimization methods in sparse learning, such as Gradient Projection, Homotopy, Iterative Shrinkage-Thresholding, Proximal Gradient, and Augmented Lagrange Multiplier methods, cannot be utilized to solve the $\ell_1$-norm based LDA objectives.

\begin{figure}[!t]
\begin{center}
\centerline{\includegraphics[height=5.5cm]{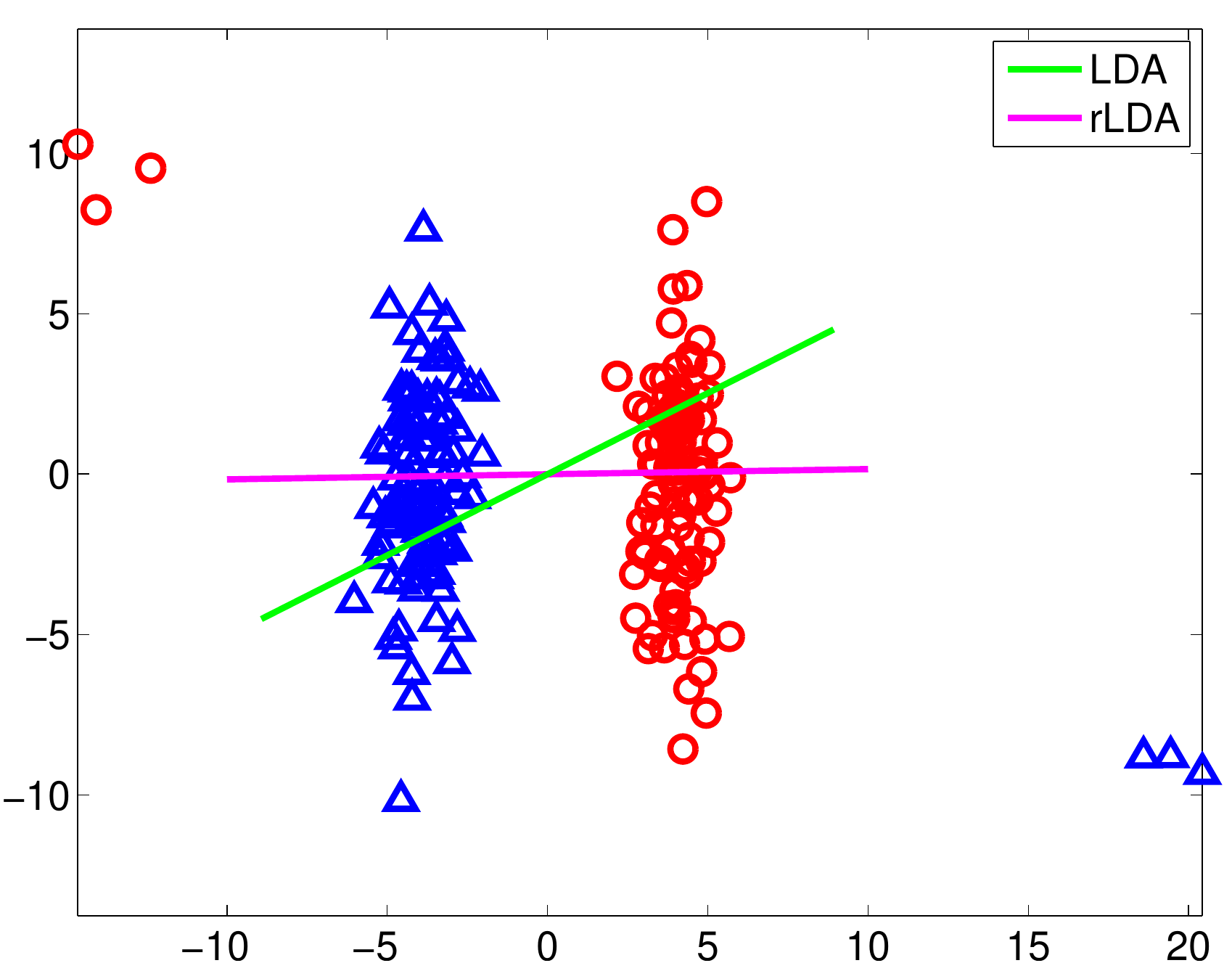}}
\caption{Two groups of data points sampled from two different Gaussian distributions are plotted with a few of outlier points. The projection directions obtained by traditional LDA and proposed robust LDA (rLDA) are calculated and visualized.}
\label{toy}
\end{center}
\end{figure}

In this paper, we propose a new robust LDA formulation that utilizes the $\ell_{1,2}$-norm in the objective to directly improve the robustness of LDA. Our new robust LDA objective imposes the $\ell_1$-norm between data points and the $\ell_2$-norm between features. The $\ell_1$-norm reduces the effect of outliers in the objective. Thus, the robustness of LDA is improved. The existing optimization algorithms cannot be applied to solve our objective. We derive a novel and efficient algorithm with rigorous theoretical analysis. In our extensive experiments, the new robust LDA outperforms other state-of-the-art related methods on the benchmark data sets with different levels of outliers.

\textbf{Notation.} Given a matrix $X \in \mathbb{R}^{d\times n}$, its $i$-th column are denoted as $x_i$.
The $\ell_{p}$-norm of a vector $v \in \mathbb{R}^n$ is defined as $||v||_p = (\sum_{i=1}^n |v_i|^p)^\frac{1}{p}$.
The $\ell_{1,2}$-norm of matrix $X$ is defined as $|| X ||_{1,2}  = \sum_j {\left\| {x_j} \right\|_2 }$. The identity matrix is denoted by $I$, the trace operator of matrix is denoted by $Tr(\cdot)$.

\section{A New Robust Formulation of Linear Discriminant Analysis}

Given the training data $X = [x_1,\cdots,x_n] \in \mathbb{R}^{d \times n}$, we denote $X_i = [x_1^i,\cdots,x_{n_i}^i] \in \mathbb{R}^{d \times n_i}(1\le i \le c)$ as the matrix of data belonging to the $i$-th class and can write $X = [X_1,\cdots,X_c]\in \mathbb{R}^{d \times n}$. LDA learns a projection matrix ${W} \in \mathbb{R}^{d\times m}$ from training data points by minimizing the distances of data points within the same class and maximizing the distances of data points between different classes simultaneously:
\begin{equation}
\label{optr0}
\mathop {\max }\limits_{W^T W = I} \frac{{\sum\limits_{i = 1}^c {\left\| {W^T (\mu_i  - \mu)} \right\|_2^2 } }}{{\sum\limits_{i = 1}^c {\sum\limits_{j = 1}^{n_i } {\left\| {W^T (x_j^i  - \mu_i )} \right\|_2^2 } } }},
\end{equation}
where $\mu_i  = \sum_{j = 1}^{n_i } {x_j^i }$ is the mean of data points in class $i$, and $\mu = \sum_{i = 1}^n {x_i }$ is the mean of all training data points.
Eq.~(\ref{optr0}) can be written as the matrix form as
\begin{equation}
\label{optr}
\mathop {\max }\limits_{W^T W = I} \frac{{Tr(W^T S_b W)}}{{Tr(W^T S_w W)}},
\end{equation}
where $S_w$ and $S_b$ are the within-class scatter matrix and between-class scatter matrix, and defined as follows:
\begin{equation}
S_w  = \sum\limits_{i = 1}^c {\sum\limits_{j = 1}^{n_i }{(x_j^i  - \mu_i )(x_j^i  - \mu_i )^T } },
\end{equation}
\begin{equation}
S_b  = \sum\limits_{i = 1}^c n_i {(\mu_i  - \mu)(\mu_i  - \mu)^T }.
\end{equation}

The trace-ratio problem in Eq.~(\ref{optr}) is somewhat difficult to be solved, traditional LDA turns to solve the following simpler ratio-trace problem:
\begin{equation}
\label{oprt}
  \mathop {\max }\limits_{W^T W = I} Tr\left(\frac{{{W^T S_b W}}}{{{W^T S_w W}}}\right).
\end{equation}
where $\frac{{A}}{{B}}$ denotes $B^{-1}A$ for simplicity.

The problem in Eq.~(\ref{oprt}) has a closed form solution, \emph{i.e.}, the $m$ eigenvectors of $S_w^{-1}S_b$ corresponding to the $m$ largest eigenvalues. Because of the rank deficiency problem of within-class scatter matrix $S_w$ for high-dimensional feature spaces or for feature spaces that have highly correlated features, researchers usually use the pseudo inverse $S_w^{+}S_b$ or discard the null space by other subspace method (\emph{e.g.} PCA) first. However, this problem doesn't exist in the trace-ratio LDA objective defined in Eq.~(\ref{optr}).
Recently, several efficient iterative algorithms were proposed to directly solve the problem in Eq.~(\ref{optr}), and have shown promising performance compared to traditional LDA \cite{WangTraceRatio,traceratioTNN09}. Thus, in this paper, we focus on the trace-ratio problem as Eq.~(\ref{optr}).

The objective function in Eq.~(\ref{optr}) uses the \textbf{squared} $\ell_2$-norms. It is widely recognized that the objective using squared $\ell_2$-norm is sensitive to the outliers. In recent sparse learning and compressive sensing techniques, the $\ell_1$-norm has been widely studied and applied to replace the squared $\ell_2$-norm in many traditional methods, such as PCA \cite{Baccini96,Gao08PCA,Ke:L1PCA,R1PCA,Kwak08PAMI,Wright09NIPSRPCA}. However, the robust formulation of LDA is not straightforward and is difficult to optimize. So far, there is no principled $\ell_1$-norm based LDA method.

In general, when the objective is to minimize the data distance or error loss, the $\ell_1$-norm objective is more robust than the squared $\ell_2$-norm objective.
Therefore, to impose the robustness on LDA, it is better to reformulate Eq.~(\ref{optr}) as a distance minimization objective.

Using the total scatter matrix $S_t  = \sum\limits_{i = 1}^n {(x_i  - \mu)(x_i  - \mu)^T }$, and with the relationship $S_t=S_w+S_b$, Eq.~(\ref{optr}) is equivalent to:
\begin{equation}
\mathop {\min }\limits_{W^T W = I} \frac{{Tr(W^T S_w W)}}{{Tr(W^T S_t W)}},
\end{equation}
which can be written as:
\begin{equation}
\label{optr1}
\mathop {\min }\limits_{W^T W = I} \frac{{\sum\limits_{i = 1}^c {\sum\limits_{j = 1}^{n_i } {\left\| {W^T (x_j^i  - \mu_i )} \right\|_2^2 } } }}{{\sum\limits_{i = 1}^n {\left\| {W^T (x_i  - \mu)} \right\|_2^2 } }}.
\end{equation}

Without loss of generality, we assume the data are centered, that is, $\mu=0$. The problem in Eq.~(\ref{optr1}) can be simplified as:
\begin{equation}
\label{optr2}
\mathop {\min }\limits_{W^T W = I} \frac{{\sum\limits_{i = 1}^c {\sum\limits_{j = 1}^{n_i } {\left\| {W^T (x_j^i  - \mu_i )} \right\|_2^2 } } }}{{\sum\limits_{i = 1}^n {\left\| {W^T x_i } \right\|_2^2 } }}.
\end{equation}

Now we still cannot simply replace the squared $\ell_2$-norm terms by $\ell_1$-norm terms in Eq.~(\ref{optr1}) or Eq.~(\ref{optr2}) to improve the robustness. Although the numerator of Eq.~(\ref{optr2}) minimizes the projection distance, the denominator of Eq.~(\ref{optr2}) still maximizes the projection distance. Thus, we have to reformulate the denominator of Eq.~(\ref{optr2}). Notice the constraint $W^T W = I$ in the problem, so
\begin{equation}
\sum\limits_{i = 1}^n {\left\| {W^T x_i } \right\|_2^2  = \sum\limits_{i = 1}^n {\left\| {x_i } \right\|_2^2 }  - \sum\limits_{i = 1}^n {\left\| {x_i  - WW^T x_i } \right\|_2^2 } },
\end{equation}
and Eq.~(\ref{optr2}) is equivalent to the following problem:
\begin{equation}
\label{optr3}
\mathop {\min }\limits_{W^T W = I} \frac{{\sum\limits_{i = 1}^c {\sum\limits_{j = 1}^{n_i } {\left\| {W^T (x_j^i  - \mu_i )} \right\|_2^2 } } }}{{\sum\limits_{i = 1}^n {\left\| {x_i } \right\|_2^2 }  - \sum\limits_{i = 1}^n {\left\| {x_i  - WW^T x_i } \right\|_2^2 } }}.
\end{equation}

Note that given the training data, $\sum_{i = 1}^n {\left\| {x_i } \right\|_2^2 }$ is a constant. $\sum_{i = 1}^n {\left\| {x_i  - WW^T x_i } \right\|_2^2 }$ is the reconstruction error.
Eq.~(\ref{optr3}) minimizes the projection distance in the numerator, and minimizes reconstruction error in the denominator. Thus, we can replace the squared $\ell_2$-norms in both of them by
$\ell_{1,2}$-norms. To be consistent, we also replace the norm in $\sum_{i = 1}^n {\left\| {x_i } \right\|_2^2 }$. Because we want to
reduce the effect of data outliers (not features) in LDA calculation, instead of using $\ell_1$-norm in objective, we use the $\ell_{1,2}$-norm, which was also used for robust principal component analysis \cite{R1PCA,rpcaICML14}. Such $\ell_{1,2}$-norm improves the robustness via using the $\ell_1$-norm between data points. The effects of outliers are reduced by the $\ell_1$-norm.  
Thus, we propose to solve the following optimization problem:
\begin{equation}
\mathop {\min }\limits_{W^T W = I} \frac{{\sum\limits_{i = 1}^c {\sum\limits_{j = 1}^{n_i } {\left\| {W^T (x_j^i  - \mu_i )} \right\|_2 } } }}{{\sum\limits_{i = 1}^n {\left\| {x_i } \right\|_2 }  - \sum\limits_{i = 1}^n {\left\| {x_i  - WW^T x_i } \right\|_2 } }},
\end{equation}
which can be written as the matrix form as:
\begin{equation}
\label{oprLDA0}
  \mathop {\min }\limits_{W^T W = I} \frac{{\sum\limits_{i = 1}^c {\left\| {W^T (X_i  - \mu_i\mathbf{1}_i^T )} \right\|_{1,2} } }}{{\left\| X \right\|_{1,2}  - \left\| {X - WW^T X} \right\|_{1,2} }}.
\end{equation}
where $\mathbf{1}_i$ is a $n_i$-dimensional vector with all elements as 1.

Further, note that $\mu_i  = \sum_{j = 1}^{n_i } {x_j^i }$ is the optimal mean under the squared $\ell_2$-norm, but is not the optimal mean under the $\ell_{1,2}$-norm. Therefore, in this paper, we also optimize the mean $\mu_i|_{i=1}^c$ for each class. The optimization problem (\ref{oprLDA0}) becomes:
\begin{equation}
\label{oprLDA}
  \mathop {\min }\limits_{W^T W = I, \mu_i|_{i=1}^c} \frac{{\sum\limits_{i = 1}^c {\left\| {W^T (X_i  - \mu_i\mathbf{1}_i^T )} \right\|_{1,2} } }}{{\left\| X \right\|_{1,2}  - \left\| {X - WW^T X} \right\|_{1,2} }}.
\end{equation}

Note that the $\ell_{1,2}$-norm is not a smooth function and the problem in Eq.~(\ref{oprLDA}) is highly non-convex. It is mentioned before that minimizing the ratio of the smooth squared $\ell_2$-norm function in Eq.~(\ref{optr1}) is not an easy problem, thus solving the problem in Eq.~(\ref{oprLDA}) is much more challenging. Because we should minimize the ratio of the non-smooth $\ell_{1,2}$-norm function, the existing $\ell_1$ minimization algorithms, such as Gradient Projection, Homotopy, Iterative Shrinkage-Thresholding, Proximal Gradient, and Augmented Lagrange Multiplier methods, cannot work here. To solve this difficult problem, in next section, we will propose a novel and efficient algorithm to optimize the ratio of $\ell_{1,2}$-norm formulations. Meanwhile, the convergence of our algorithm is theoretically guaranteed.

\section{Optimization Algorithm}

\subsection{Algorithm to A General Problem}

Before solving the problem in Eq.~(\ref{oprLDA}), let's consider a more general problem as follows:
\begin{equation}\label{opfp}
      \min_{{v}\in \mathcal{C}}\;
      \frac{f({v})}
      {g({v}) } \;,
\end{equation}
We suppose the problem (\ref{opfp}) is lower bounded. The algorithm to solve Eq.(\ref{opfp}) is described in Algorithm~\ref{alg}. In the following we will prove the algorithm converges to the globally optimal solution to the problem (\ref{opfp}), and the convergence rate is quadratic.

\begin{algorithm}[!t]
\caption{The algorithm to solve the problem (\ref{opfp}).}
\label{alg}
\begin{algorithmic}
\STATE Initialize ${v} \in \mathcal{C}$.
\REPEAT
\STATE \textbf{1.} Calculate $\lambda = \frac{f(v)} {g(v)}$.
\STATE \textbf{2.} Update $v$ by solving the following problem:
 \begin{equation}\label{opdiff}
 {v} = \arg \min_{{v}\in \mathcal{C}} f({v}) - \lambda g({v})
 \end{equation}
\UNTIL {Converges}
\end{algorithmic}
\end{algorithm}

\begin{theorem}\label{prop0}
  Algorithm~\ref{alg} decreases the objective value of the problem (\ref{opfp}) in each iteration.
\end{theorem}
\begin{proof}
    In each iteration of Algorithm~\ref{alg}, suppose the updated $v$ in step 2 is $\tilde v$. Then we have $f({\tilde v}) - \lambda g({\tilde v}) \le f({v}) - \lambda g({v})$. According to step 1, we know  $f({v}) - \lambda g({v})=0$. Thus $f({\tilde v}) - \lambda g({\tilde v}) \le 0$, which indicates $\frac{f({\tilde v})}{g({\tilde v})} \le \lambda = \frac{f({v})}{g({v})}$ as $g(\tilde v)\ge 0$.
\end{proof}

Since the problem (\ref{opfp}) is lower bounded, the Algorithm~\ref{alg} will converge according to Theorem \ref{prop0}. The following two theorems reveal
that the Algorithm~\ref{alg} will converge to the globally optimal solution with quadratic convergence rate.

\begin{theorem}\label{theorem:equiv}
  The converged $v$ in Algorithm~\ref{alg} is the globally optimal solution to the problem (\ref{opfp}).
\end{theorem}
\begin{proof}
    Suppose ${v}^*$ is the converged solution $v$ in Algorithm~\ref{alg}, and $\lambda^*$ is the converged objective value. According to step 1, the following holds:
    \begin{equation}\label{}
    \frac{f({v}^*)}
      {g({v}^*) } = \lambda^*.
    \end{equation}
    Thus according to step 2, we have, $\forall\; {v} \in \mathcal{C}$,
    \begin{equation}\label{vl0}
    f({v}) - \lambda^* g({v}) \ge f({v^*}) - \lambda^* g({v^*}) = 0.
    \end{equation}
    Note that $g(v)\ge 0$, Eq.(\ref{vl0}) means $\forall\; {v} \in \mathcal{C}$, $\frac{f({v})}
      {g({v}) } \ge \lambda^*$. That is to say,
    \begin{equation}\label{}
    \frac{f({v}^*)}
      {g({v}^*) } = \lambda^* = \min_{{v}\in \mathcal{C}}\;
      \frac{f({v})}
      {g({v}) }.
    \end{equation}
    So ${v}^*$ is the globally optimal solution to the problem (\ref{opfp}).
\end{proof}

\begin{theorem}\label{prop1}
  The convergence rate of Algorithm~\ref{alg} is quadratic.
\end{theorem}
\begin{proof}
Define a function as follows:
\begin{equation}\label{}
h(\lambda) = \min_{{v}\in \mathcal{C}}\; f({v}) - \lambda g({v}).
\end{equation}
According to Algorithm~\ref{alg} we know, the converged ${\lambda}^*$ is the root of $h(\lambda)$, that is, $h(\lambda^*)=0$.

In each iteration of Algorithm~\ref{alg}, suppose the updated $v$ is $\tilde v$ in step 2.
According to step 2, $h(\lambda) = f({\tilde v}) - \lambda g({\tilde v})$. Thus $h'(\lambda) = -g({\tilde v})$.

In Newton's method, the updated solution should be $\tilde \lambda = \lambda - \frac{h(\lambda)}{h'(\lambda)} = \lambda - \frac{f({\tilde v}) - \lambda g({\tilde v})}{-g({\tilde v})} = \frac{f({\tilde v} )} {g({\tilde v} )}$, which is exactly the step 1 in the next iteration in Algorithm~\ref{alg}. Namely, Algorithm~\ref{alg} is a Newton's method to find the root of function $h(\lambda)$, and thus the convergence
rate is quadratic.
\end{proof}
Theorem~\ref{prop1} indicates that Algorithm~\ref{alg} converges very fast, \ie, the difference between the current objective value and the optimal objective value is smaller than $\frac{1}{c^{2^t}}$($c>1$ is a certain constant) at the \kth[t] iteration.

From the above analysis we know, if we can find the globally optimal solution to the problem (\ref{opdiff}) in step 2 of Algorithm~\ref{alg}, we can find the globally optimal solution to the problem (\ref{opfp}) by Algorithm~\ref{alg}. In some cases, it is difficult to find the globally optimal solution to the problem (\ref{opdiff}).

Before analyzing the property of Algorithm~\ref{alg} in these cases, let's recall the method of Lagrange multipliers for the following constrained problem:
\begin{equation}\label{opf}
      \min_{{v}\in \mathcal{C}}\; f(v).
\end{equation}
First, suppose the constraint ${v}\in \mathcal{C}$ can be written as $p(v)\le 0$ and $q(v)=0$, where $p(v)$ and $q(v)$ are vector output function to encode multiple
constraints. Thus the problem (\ref{opf}) is equivalent to
\begin{equation}\label{opf1}
\mathop {\min }\limits_v \mathop {\max }\limits_{\alpha  \ge 0,\beta } f(v) + \alpha ^T p(v) + \beta ^T q(v),
\end{equation}
and the Lagrange function of problem (\ref{opf}) is defined as $\mathcal{L}(v,\alpha ,\beta ) = f(v) + \alpha ^T p(v) + \beta ^T q(v)$, where $(\alpha ,\beta)$ is called the Lagrange multipliers.
The necessary condition for $v$ being an optimal solution to the problem (\ref{opf}) is that, there \textbf{exists} a $(\alpha ,\beta)$ such that $(v, \alpha ,\beta)$ is a stationary point to the Lagrange function $\mathcal{L}(v,\alpha ,\beta )$.

\begin{theorem}\label{ratiolocal}
  If the updated $v$ in step 2 of Algorithm~\ref{alg} is a stationary point of the problem (\ref{opdiff}), the converged solution in Algorithm~\ref{alg} is a stationary point of problem (\ref{opfp}).
\end{theorem}
\begin{proof}
The Lagrange function of problem (\ref{opdiff}) is
\begin{equation}\label{}
\mathcal{L}_1 (v,\alpha _1 ,\beta _1 ) = f(v) - \lambda g(v) + \alpha _1^T p(v) + \beta _1^T q(v).
\end{equation}
Suppose the converged solution in Algorithm~\ref{alg} is ${v}^*$. If ${v}^*$ is a stationary point of the problem (\ref{opdiff}), we have
\begin{equation}\label{}
f'(v^{\rm{*}} ) - \frac{{f(v^* )}}{{g(v^* )}}g'(v^* ) + \alpha _1^T p'(v^* ) + \beta _1^T q'(v^* ) = 0,
\end{equation}
which can be written as
\[
\frac{{g(v^* )f'(v^* ) - f(v^* )g'(v^* )}}{{g^2 (v^* )}} + \frac{{\alpha _1^T }}{{g(v^* )}}p'(v^* ) + \frac{{\beta _1^T }}{{g(v^* )}}q'(v^* ) = 0.
\]
Let $\alpha _2  = \frac{{\alpha _1 }}{{g(v^* )}}$ and $\beta _2  = \frac{{\beta _1 }}{{g(v^* )}}$, then we have
\begin{equation}\label{}
\left. {\left( {\frac{{f(v)}}{{g(v)}}} \right)^\prime  } \right|_{v = v^* }  + \alpha _2^T p'(v^* ) + \beta _2^T q'(v^* ) = 0.
\end{equation}
Note that $\alpha_1\ge0$ and $g(v^* ) \ge 0$, so $\alpha_2\ge0$. Therefore, ${v}^*$ is a stationary point to the Lagrange function of problem (\ref{opfp}) as follows:
\begin{equation}\label{}
\mathcal{L}_2 (v,\alpha_2 ,\beta_2 ) = \frac{{f(v)}}{{g(v)}} + \alpha _2^T p(v) + \beta _2^T q(v),
\end{equation}
which completes the proof.
\end{proof}

\begin{algorithm}[!t]
\caption{An efficient iterative algorithm to solve the rLDA optimization problem (\ref{oprLDA}).}
\label{alg1}
\begin{algorithmic}
\STATE {\bfseries Input:} $X=[x_1,x_2,\cdots,x_n] = [X_1,X_2,\cdots,X_c] \in \mathbb{R}^{d \times n}$, where $X_i = [x_1^i,x_2^i,\cdots,x_{n_i}^i] \in \mathbb{R}^{d \times n_i}(1\le i \le c)$ is the data matrix belongs to the $i$-th class, $X$ is a centered data matrix with zero mean.
\STATE {\bfseries Output:} The projection matrix ${W} \in \mathbb{R}^{d\times m}$.
\STATE Initialize ${W} \in \mathbb{R}^{d \times m}$ such that $W^T W = I$. \\ Initialize $\mu_i  = \sum_{j = 1}^{n_i } {x_j^i }$ for each $i$.
\REPEAT
\STATE \textbf{1.} Calculate $\lambda  = \frac{{\sum\limits_{i = 1}^c {\left\| {W^T (X_i  - \mu_i\mathbf{1}_i^T )} \right\|_{1,2} } }}{{\left\| X \right\|_{1,2}  - \left\| {X - W W^T X} \right\|_{1,2} }}$.
\STATE \textbf{2.} Calculate the diagonal matrix ${D}_i(1\le i \le c) \in \mathbb{R}^{n_i \times n_i}$, where the $j$-th diagonal element is $\frac{1}{{2\left\| {W^T (x_j^i  - \mu_i) } \right\|_2 }}$.
\STATE \textbf{3.} Calculate the diagonal matrix ${ D}\in \mathbb{R}^{n \times n}$, where the $i$-th diagonal element is $\frac{1}{{2\left\| {x_i - W W^T x_i } \right\|_2 }}$.
\STATE \textbf{4.} Calculate the matrix $A$ by \\$A = \sum\limits_{i = 1}^c {X_i (D_i  - \frac{1}{{\mathbf{1}_i^T D_i \mathbf{1}_i }}D_i \mathbf{1}_i \mathbf{1}_i^T D_i )X_i^T }  - \lambda XDX^T$.
\STATE \textbf{5.} Update $\mu_i(1\le i \le c)$ by $\mu_i = \frac{1}{\mathbf{1}_i^TD_i \mathbf{1}_i} X_i D_i\mathbf{1}_i$. \\ Update $W$ by
$W  = \mathop {\min }_{W^T W = I} Tr(W^T A W)$, \emph{i.e.}, $W$ is formed by the $m$ eigenvectors of $A$ corresponding to the $m$ smallest eigenvalues.
\UNTIL {Converges}
\end{algorithmic}
\end{algorithm}

\subsection{Algorithm to the robust LDA problem (\ref{oprLDA})}
\label{rLDAalg}

In order to solve the robust LDA problem (\ref{oprLDA}), we need to solve the following problem according to the Algorithm~\ref{alg}:
\begin{equation}\label{}
\begin{array}{l}
 \mathop {\min }\limits_{W^T W = I,\left. {\mu _i } \right|_{i = 1}^c } \sum\limits_{i = 1}^c {\left\| {W^T (X_i  - \mu _i 1_i^T )} \right\|_{1,2} }  \\
 \quad\quad\quad\quad\quad\quad - \lambda \left( {\left\| X \right\|_{1,2}  - \left\| {X - WW^T X} \right\|_{1,2} } \right), \\
 \end{array}
\end{equation}
which is equivalent to
\begin{equation}\label{opWu21}
\begin{array}{l}
 \mathop {\min }\limits_{W^T W = I,\left. {\mu _i } \right|_{i = 1}^c } \sum\limits_{i = 1}^c {\left\| {W^T (X_i  - \mu _i \mathbf{1}_i^T )} \right\|_{1,2} }  \\
 \quad \quad \quad \quad\quad\quad  + \lambda \left\| {X - WW^T X} \right\|_{1,2}.  \\
 \end{array}\quad
\end{equation}
Using the re-weighted method as shown in \cite{RFSnips10,rpcaICML14}, we can iteratively solve the following problem to obtain a stationary point to the
problem (\ref{opWu21}):
\begin{equation}\label{}
\begin{array}{l}
 \mathop {\min }\limits_{W^T W = I,\left. {\mu _i } \right|_{i = 1}^c } \sum\limits_{i = 1}^c {Tr(W^T (X_i  - \mu _i \mathbf{1}_i^T )D_i (X_i  - \mu _i \mathbf{1}_i^T )^T W)}  \\
 \quad\quad\quad\quad\quad\quad + \lambda Tr((X - WW^T X)D(X - WW^T X)^T ), \nonumber \\
 \end{array}
\end{equation}
where ${D}_i(1\le i \le c)$ is a diagonal matrix with the $j$-th diagonal element as $\frac{1}{{2\left\| {W^T (x_j^i  - \mu_i) } \right\|_2 }}$, and ${ D}$ is a diagonal matrix with the the $i$-th diagonal element as $\frac{1}{{2\left\| {x_i - W W^T x_i } \right\|_2 }}$\footnote{In practice\cite{RFSnips10}, $\frac{1}{2\left\|v\right\|_2}$ can be calculated as $\frac{1}{2\sqrt{\left\|v\right\|_2^2+\varepsilon}} (\varepsilon\rightarrow 0)$, and the derived algorithm is to minimize $\sqrt{\left\|v\right\|_2^2+\varepsilon}$ instead of ${\left\|v\right\|_2}$.}.
The above problem is equivalent to
\begin{equation}\label{opWu21rew}
\begin{array}{l}
 \mathop {\min }\limits_{\scriptstyle W^T W = I, \hfill \atop
  \scriptstyle \left. {\mu _i } \right|_{i = 1}^c  \hfill} \sum\limits_{i = 1}^c {Tr(W^T (X_i  - \mu _i 1_i^T )D_i (X_i  - \mu _i 1_i^T )^T W)}  \\
  \quad\quad\quad\quad - \lambda Tr(W^T XDX^T W). \\
 \end{array}
\end{equation}
The problem (\ref{opWu21rew}) can be solved with closed form solution. For each $i$, by taking the derivative of Eq.(\ref{opWu21rew}) w.r.t. $\mu_i$ to zero, we have
\begin{equation}\label{solmu}
\mu_i = \frac{1}{\mathbf{1}_i^TD_i \mathbf{1}_i} X_i D_i\mathbf{1}_i.
\end{equation}
By replacing the $\mu_i$ in Eq.(\ref{opWu21rew}) with Eq.(\ref{solmu}), the problem (\ref{opWu21rew}) becomes
\begin{equation}\label{opWAW}
\mathop {\min }\limits_{W^T W = I} Tr(W^T AW),
\end{equation}
where $A = \sum\limits_{i = 1}^c {X_i (D_i  - \frac{1}{{\mathbf{1}_i^T D_i \mathbf{1}_i }}D_i \mathbf{1}_i \mathbf{1}_i^T D_i )X_i^T }  - \lambda XDX^T$.
The optimal solution $W$ to the problem (\ref{opWAW}) is formed by the eigenvectors of $A$ corresponding to the smallest eigenvalues.

Based on Algorithm~\ref{alg}, we propose an algorithm to solve the robust LDA problem (\ref{oprLDA}), which is described in Algorithm \ref{alg1}.
From Algorithm \ref{alg1}, we can see that our algorithm can be easily implemented without using any other additional optimization toolbox. We will prove that the algorithm decreases the objective value in each iteration and thus the convergence is guaranteed. The algorithm has closed form solution in each iteration, and the algorithm converges very fast. In our extensive empirical studies on nine benchmark data sets, the algorithm always converges within 20 iterations (convergence criterion: the objective function value difference between two iterations less than $10^{-6}$). More importantly, according to Theorem \ref{ratiolocal}, the converged solution is a stationary point of problem (\ref{oprLDA}) since we find a stationary point
of problem (\ref{opWu21}) in each iteration of the algorithm. Therefore, the quality of the solution found by Algorithm \ref{alg1} is theoretically guaranteed.

\subsection{Convergence Analysis of Algorithm \ref{alg1}}

To prove the convergence of the above algorithm, we first prove two lemmas.
\begin{lemma}
\label{lem1}
For any column-orthogonal matrix $W$ such that $W^T W = I$, we have the following inequality:
\[\left\| X \right\|_{1,2}  - \left\| {X - W W^T X} \right\|_{1,2}  \ge 0.\]
\end{lemma}
\begin{proof}
For every $x_i$, according to $W^T W = I$, we have
\begin{equation}
 \left\| {x_i } \right\|_2^2  - \left\| {x_i  - W W^T x_i } \right\|_2^2    = Tr(W^T x_i x_i^T W )
  \ge 0.\nonumber
\end{equation}
So we have $\sum\limits_{i = 1}^n {\left( {\left\| {x_i } \right\|_2  - \left\| {x_i  - W W^T x_i } \right\|_2 } \right)}  \ge 0$. That is to say,
$\left\| X \right\|_{1,2}  - \left\| {X - W W^T X} \right\|_{1,2}  \ge 0$.
\end{proof}

\begin{lemma}
\label{lem3}
For any vectors $v_i$ and $\tilde v_i$, we have:
\[\sum_i {\frac{{\left\| {\tilde v_i } \right\|_2^2 }}{{2\left\| {v_i } \right\|_2 }}}  \le \sum_i {\frac{{\left\| v_i \right\|_2^2 }}{{2\left\| {v_i } \right\|_2 }}}  \Rightarrow \sum_i {\left\| {\tilde v_i } \right\|_2 }  \le \sum_i {\left\| {v_i } \right\|_2 }.
\]
\end{lemma}
\begin{proof}
Obviously, for every $i$ we have $- \left( {\left\| {\tilde v_i } \right\|_2  - \left\| {v_i } \right\|_2 } \right)^2  \le 0$, which indicates
$\left\| {\tilde v_i } \right\|_2  - \left\| { v_i } \right\|_2  \le \frac{{\left\| {\tilde v_i } \right\|_2^2 }}{{2\left\| { v_i } \right\|_2 }} - \frac{{\left\| { v_i } \right\|_2^2 }}{{2\left\| { v_i } \right\|_2 }}$ for every $i$, and thus $\sum_i {\left\| {\tilde v_i } \right\|_2 }  - \sum_i {\left\| { v_i } \right\|_2 }  \le \sum_i {\frac{{\left\| {\tilde v_i } \right\|_2^2 }}{{2\left\| { v_i } \right\|_2 }}}  - \sum_i {\frac{{\left\| { v_i } \right\|_2^2 }}{{2\left\| { v_i } \right\|_2 }}}$. Then if $\sum_i {\frac{{\left\| {\tilde v_i } \right\|_2^2 }}{{2\left\| { v_i } \right\|_2 }}}  \le \sum_i {\frac{{\left\| { v_i} \right\|_2^2 }}{{2\left\| { v_i } \right\|_2 }}}$, we have $\sum_i {\left\| {\tilde v_i } \right\|_2 }  - \sum_i {\left\| { v_i } \right\|_2 } \le 0$, which completes the proof.
\end{proof}

Now we prove the main result as the following theorem:
\begin{theorem}
\label{thm1}
  The Algorithm \ref{alg1} decreases the objective value of problem~(\ref{oprLDA}) in each iteration until converges.
\end{theorem}
\begin{proof}
 In each iteration of Algorithm \ref{alg1}, suppose the updated $W$ and $\mu_i(1\le i \le c)$ in step 5 are $\tilde W$ and $\tilde \mu_i(1\le i \le c)$, respectively. According to the analysis in Section \ref{rLDAalg}, we know $\tilde W, \tilde \mu_i(1\le i \le c)$ is the optimal solution to the problem (\ref{opWu21rew}). So we have
\begin{align*}
& \sum\limits_{i = 1}^c {Tr(\tilde W^T (X_i  - \tilde \mu_i\mathbf{1}_i^T )D_i (X_i  - \tilde \mu_i\mathbf{1}_i^T )^T \tilde W) }  \\
&  - \lambda Tr(\tilde W^T X D X^T \tilde W)  \\
  \le & \sum\limits_{i = 1}^c {Tr( W^T (X_i  -  \mu_i\mathbf{1}_i^T )D_i (X_i  -  \mu_i\mathbf{1}_i^T )^T  W) }  \\
&  - \lambda Tr( W^T X D X^T  W).
\end{align*}
Note that the following equality holds: $Tr(W^T XD X^T W) = Tr(XD X^T)  - Tr(( {X - WW^T X} )D ( {X - WW^T X} )^T)$. Then we have
\begin{align*}
& \sum\limits_{i = 1}^c {Tr(\tilde W^T (X_i  - \tilde \mu_i\mathbf{1}_i^T )D_i (X_i  - \tilde \mu_i\mathbf{1}_i^T )^T \tilde W) } \\
&  + \lambda Tr(( {X - \tilde W \tilde W^T X} ) D ( {X - \tilde W \tilde W^T X})^T)  \\
  \le& \sum\limits_{i = 1}^c {Tr( W^T (X_i  -  \mu_i\mathbf{1}_i^T )D_i (X_i  -  \mu_i\mathbf{1}_i^T )^T  W) } \\
&  + \lambda Tr(( {X -  W  W^T X} ) D ( {X -  W  W^T X})^T)  \\
\Rightarrow &\sum\limits_{i = 1}^c {\sum\limits_{j = 1}^{n_i } {\frac{\left\| {\tilde W^T (x_j^i  - \tilde \mu_i )} \right\|_2^2 }{2\left\| {W^T (x_j^i  - \mu_i )} \right\|_2 }} }  + \sum\limits_{j = 1}^n {\frac{\left\| {\lambda (x_j  - \tilde W \tilde W^T x_j )} \right\|_2^2 }{2\left\| {\lambda (x_j  - W W^T x_j )} \right\|_2 }}  \\
  \le &\sum\limits_{i = 1}^c {\sum\limits_{j = 1}^{n_i } {\frac{\left\| { W^T (x_j^i  - \mu_i )} \right\|_2^2 }{2\left\| {W^T (x_j^i  - \mu_i )} \right\|_2 }} }  + \sum\limits_{j = 1}^n {\frac{\left\| {\lambda (x_j  -  W  W^T x_j )} \right\|_2^2 }{2\left\| {\lambda (x_j  - W W^T x_j )} \right\|_2 }}.
\end{align*}

According to Lemma \ref{lem3}, we have the following inequalities:
\begin{align*}
& \sum\limits_{i = 1}^c {\sum\limits_{j = 1}^{n_i } {\left\| {\tilde W^T (x_j^i  - \tilde \mu_i )} \right\|_2 } }  + \sum\limits_{j = 1}^n {\left\| {\lambda (x_j  - \tilde W \tilde W^T x_j )} \right\|_2 }  \\
  \le & \sum\limits_{i = 1}^c {\sum\limits_{j = 1}^{n_i } {\left\| { W^T (x_j^i  -  \mu_i )} \right\|_2 } }  + \sum\limits_{j = 1}^n {\left\| {\lambda (x_j  -  W  W^T x_j )} \right\|_2 }  \\
  \Rightarrow & \sum\limits_{i = 1}^c {\left\| {\tilde W^T (X_i  - \tilde \mu_i\mathbf{1}_i^T )} \right\|_{1,2} }  + \lambda \left\| {X - \tilde W\tilde W^T X} \right\|_{1,2}  \\
  \le& \sum\limits_{i = 1}^c {\left\| { W^T (X_i  -  \mu_i\mathbf{1}_i^T )} \right\|_{1,2} }  + \lambda \left\| {X -  W W^T X} \right\|_{1,2},
\end{align*}
which is equivalent to
\begin{align}
\label{ineq1}
& \sum\limits_{i = 1}^c {\left\| {\tilde W^T (X_i  - \tilde \mu_i\mathbf{1}_i^T )} \right\|_{1,2} } + \nonumber\\
&  \lambda \left\| {X - \tilde W\tilde W^T X} \right\|_{1,2}  - \lambda \left\| X \right\|_{1,2}  \nonumber\\
  \le & \sum\limits_{i = 1}^c {\left\| { W^T (X_i  -  \mu_i\mathbf{1}_i^T )} \right\|_{1,2} } + \nonumber\\
&  \lambda \left\| {X -  W W^T X} \right\|_{1,2}  - \lambda \left\| X \right\|_{1,2}.
\end{align}
According to step 1 in Algorithm \ref{alg1}, we know that
\begin{equation}
\label{lamda}
  \lambda  = \frac{{\sum\limits_{i = 1}^c {\left\| {W^T (X_i  - \mu_i\mathbf{1}_i^T )} \right\|_{1,2} } }}{{\left\| X \right\|_{1,2}  - \left\| {X - W W^T X} \right\|_{1,2} }},
\end{equation}
which indicates that
\begin{equation}
\label{eq1}
\begin{array}{l}
 \sum\limits_{i = 1}^c {\left\| {W^T (X_i  - \mu_i\mathbf{1}_i^T )} \right\|_{1,2} }  \\
  + \lambda \left\| {X - W W^T X} \right\|_{1,2}  - \lambda \left\| X \right\|_{1,2}  = 0. \\
 \end{array}
\end{equation}
According to Eq.~(\ref{ineq1}) and Eq.~(\ref{eq1}), we arrive at
\begin{equation}
\label{ineq2}
\begin{array}{l}
 \sum\limits_{i = 1}^c {\left\| {\tilde W^T (X_i  - \tilde \mu_i\mathbf{1}_i^T )} \right\|_{1,2} }  \\
  + \lambda \left\| {X - \tilde W \tilde W^T X} \right\|_{1,2}  - \lambda \left\| X \right\|_{1,2} \le 0. \\
 \end{array}
\end{equation}
Note that $\left\| X \right\|_{1,2}  - \left\| {X - \tilde W\tilde W^T X} \right\|_{1,2} \ge 0$ according to Lemma \ref{lem1}, So Eq.~(\ref{ineq2}) results in the following inequality:
\begin{equation}
\label{ineq3}
\frac{{\sum\limits_{i = 1}^c {\left\| {\tilde W^T (X_i  - \tilde \mu_i\mathbf{1}_i^T )} \right\|_{1,2} } }}{{\left\| X \right\|_{1,2}  - \left\| {X - \tilde W \tilde W^T X} \right\|_{1,2} }} \le \lambda.
\end{equation}
Based on Eq.~(\ref{lamda}) and Eq.~(\ref{ineq3}), we arrive at
\begin{eqnarray}
&&\frac{{\sum\limits_{i = 1}^c {\left\| {\tilde W^T (X_i  - \tilde \mu_i\mathbf{1}_i^T )} \right\|_{1,2} } }}{{\left\| X \right\|_{1,2}  - \left\| {X - \tilde W \tilde W^T X} \right\|_{1,2} }}  \nonumber\\
&&\le \frac{{\sum\limits_{i = 1}^c {\left\| { W^T (X_i  -  \mu_i\mathbf{1}_i^T )} \right\|_{1,2} } }}{{\left\| X \right\|_{1,2}  - \left\| {X -  W W^T X} \right\|_{1,2} }}.
\end{eqnarray}
Note that the equalities in the above equations hold only when the algorithm converges.
Therefore, the Algorithm \ref{alg1} will decrease the objective value in each iteration until the algorithm converges.

On the other hand, the denominator in Eq.~(\ref{oprLDA}) is not smaller than $0$ according to Lemma \ref{lem1}, thus the problem (\ref{oprLDA}) has a lower bound $0$. Therefore, Algorithm \ref{alg1} will converge.
\end{proof}

\begin{figure*}[!t]
  \centering
  \subfigure[Vehicle data set.]{\label{fig_acc_var_vehicle}\includegraphics[width=0.4\linewidth]{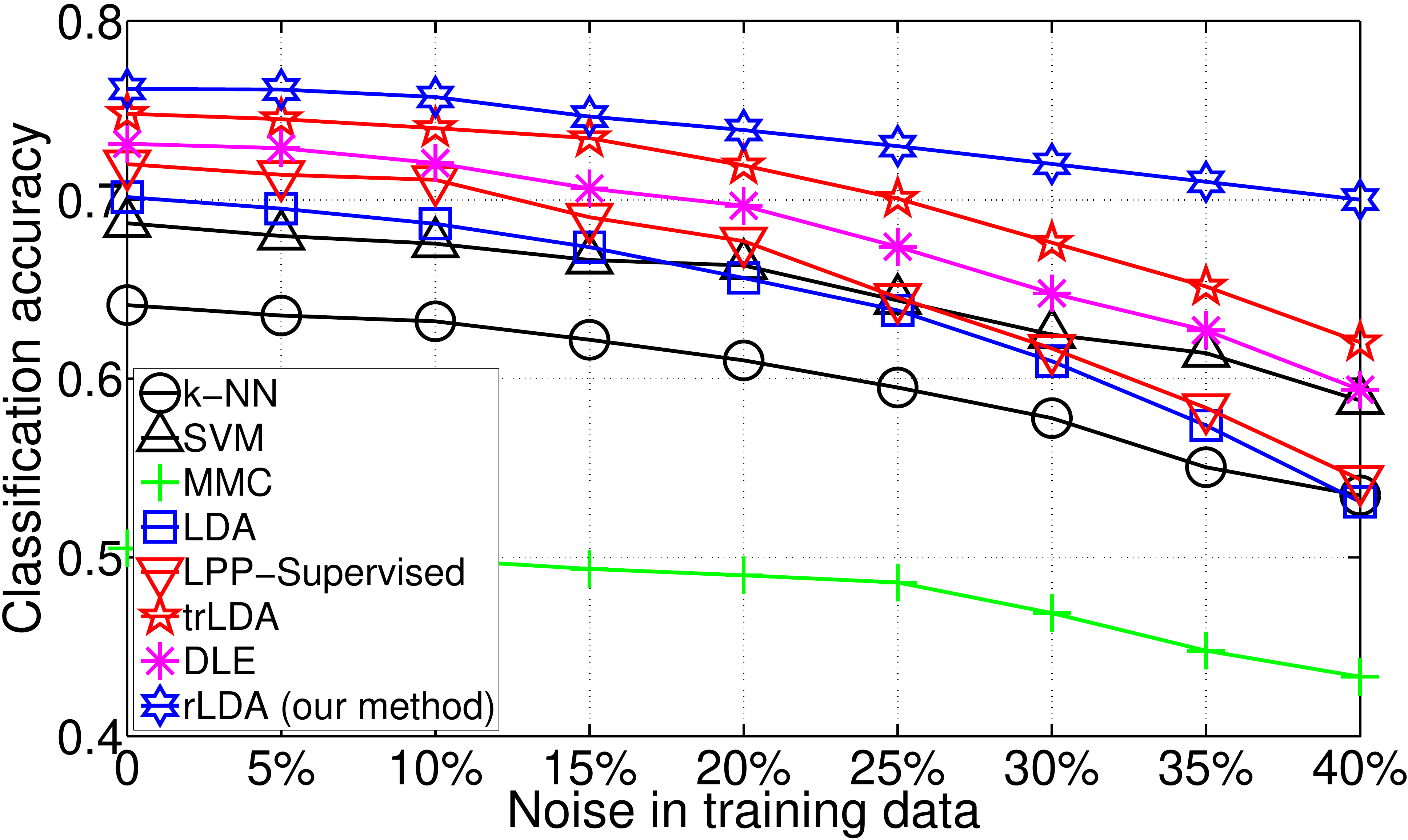}}\qquad
\subfigure[Dermatology data set.]{\label{fig_acc_var_dermatology}\includegraphics[width=0.4\linewidth]{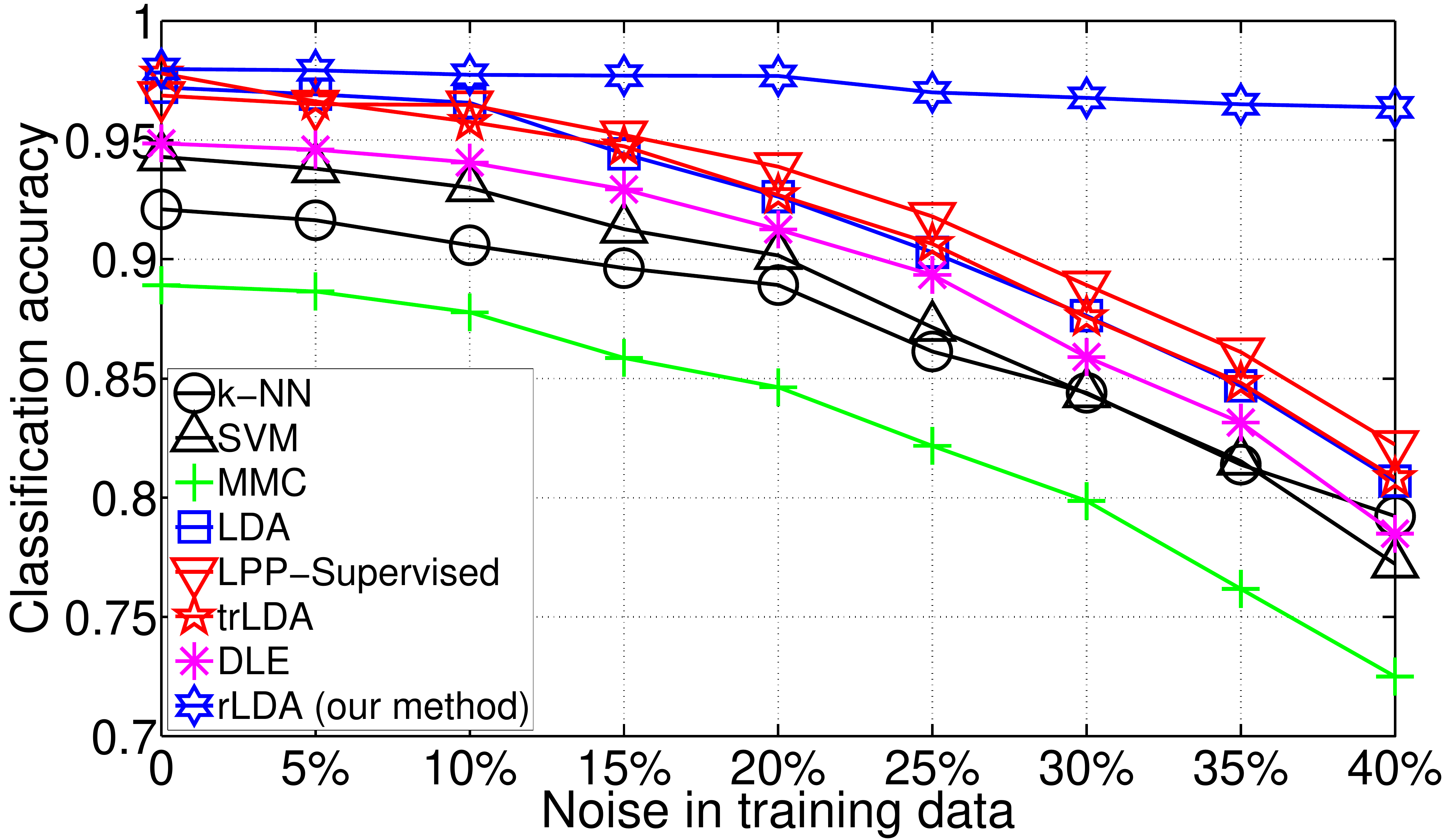}}\qquad
  \subfigure[Coil-20 data set.]{\label{fig:acc_var_coil}\includegraphics[width=0.4\linewidth]{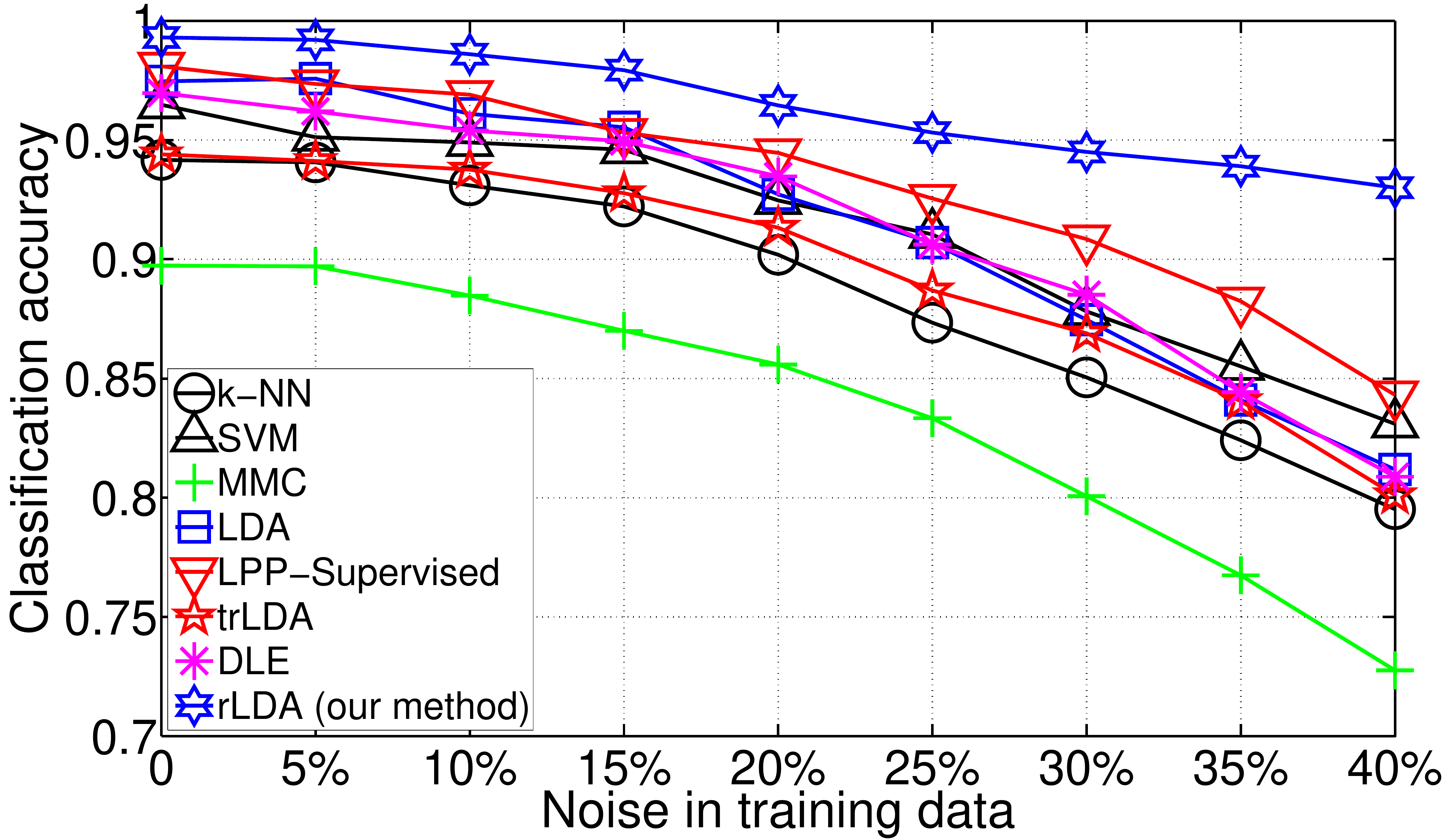}}\qquad
  \subfigure[ORL face data set.]{\label{fig:acc_var_orl}\includegraphics[width=0.4\linewidth]{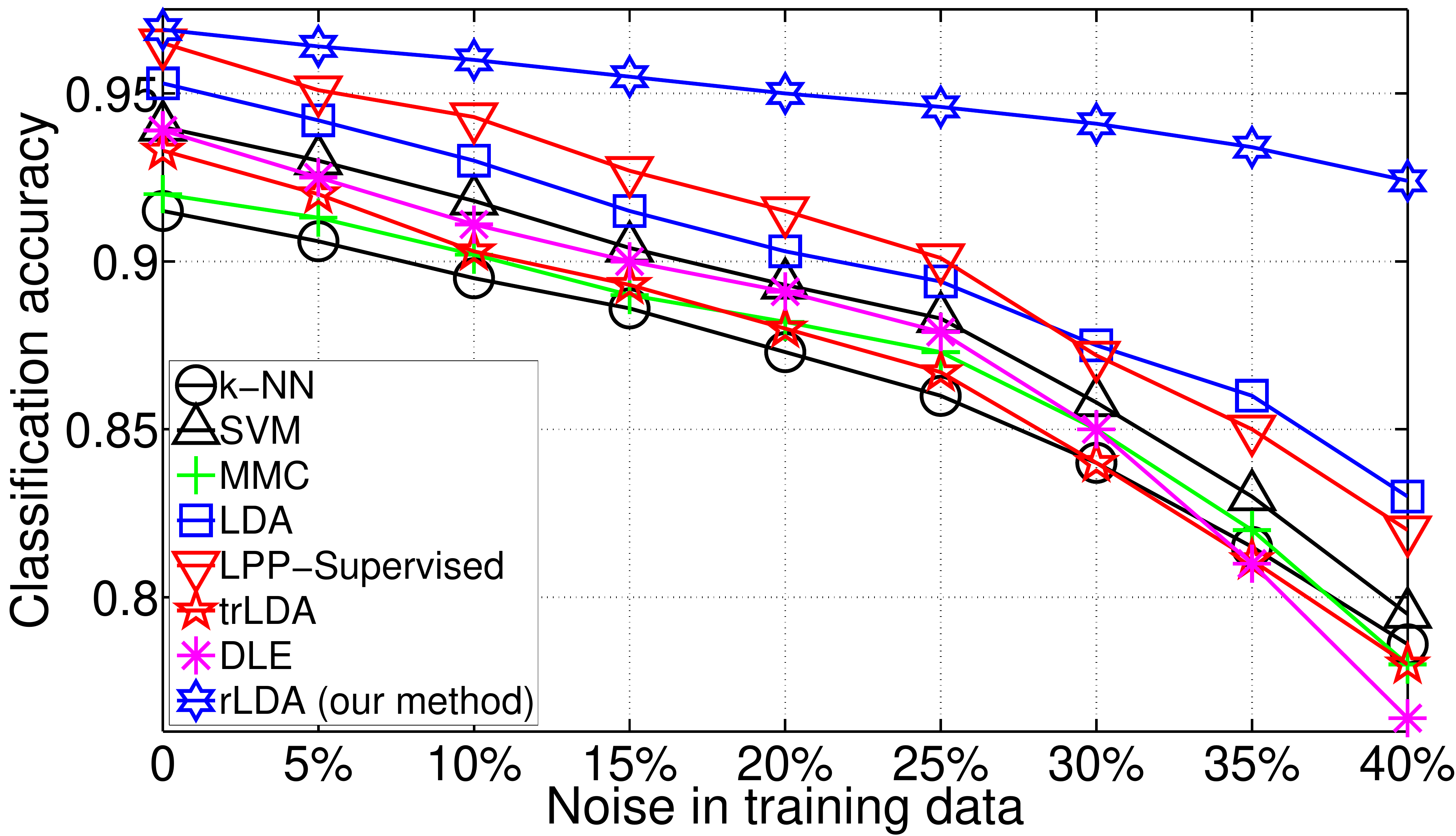}}\qquad
  \subfigure[Yale face data set.]{\label{fig:acc_var_yale}\includegraphics[width=0.4\linewidth]{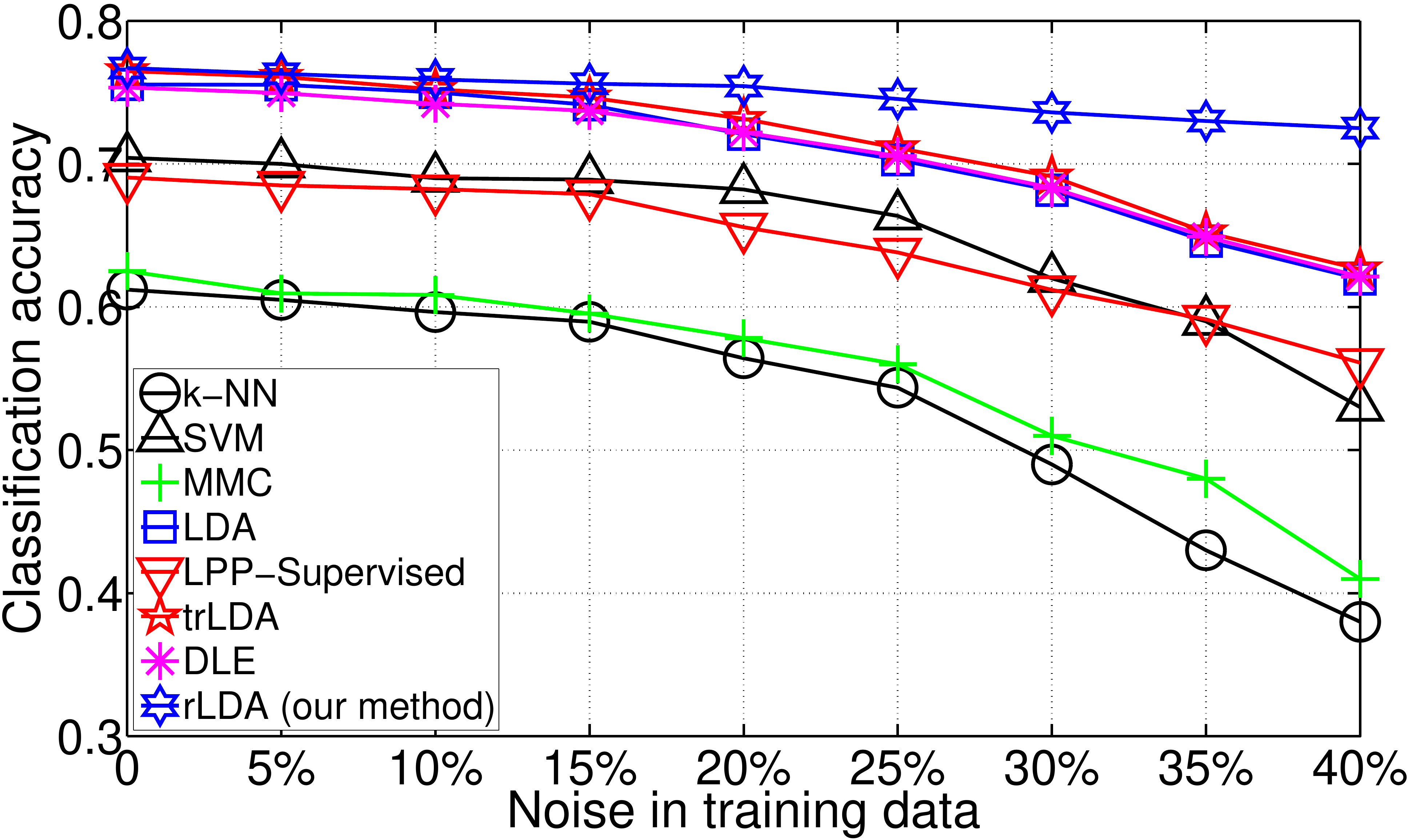}}\qquad
  \subfigure[PIE face data set.]{\label{fig:acc_var_pie}\includegraphics[width=0.4\linewidth]{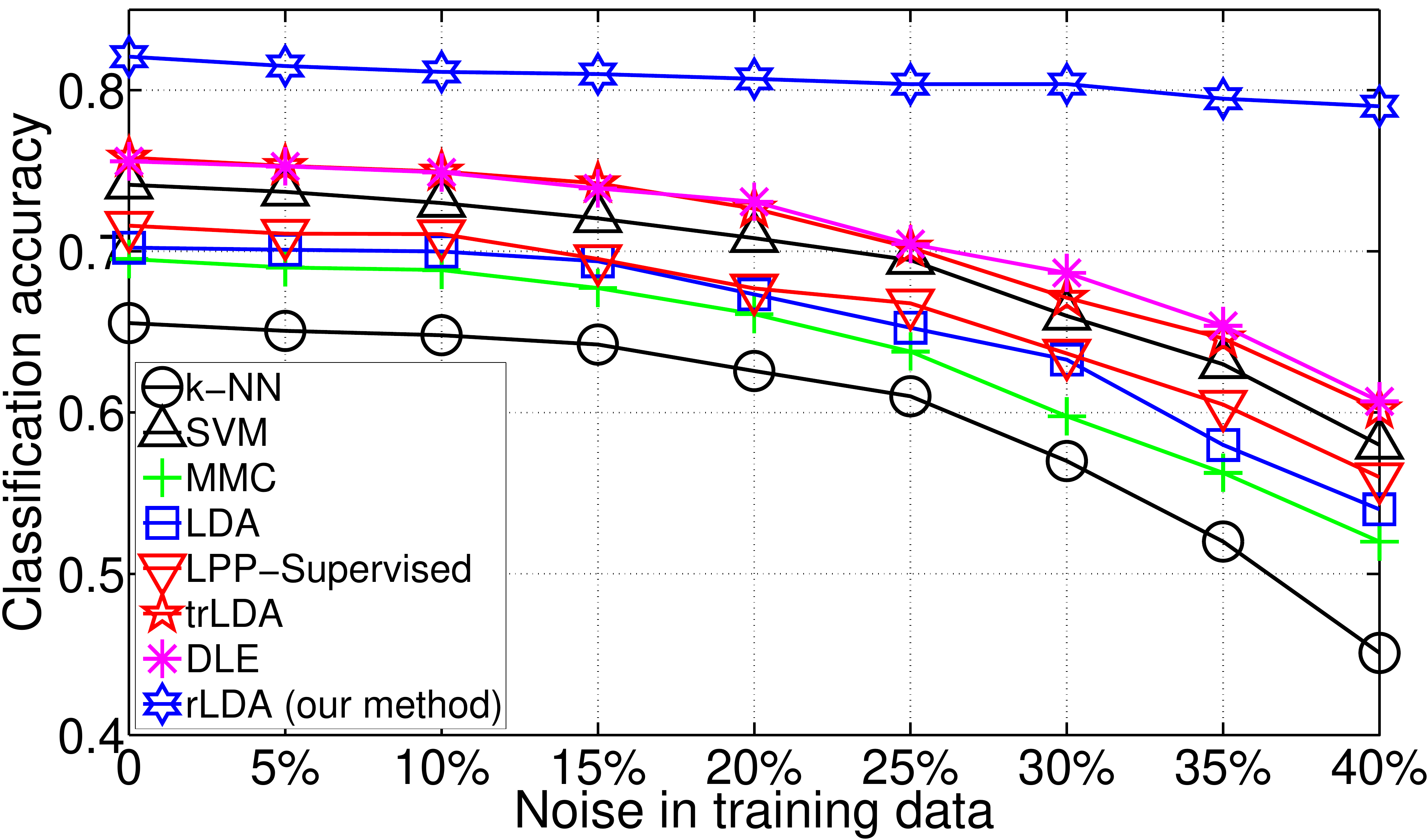}}\qquad
   \subfigure[Reuters data set.]{\label{fig:acc_var_reuters}\includegraphics[width=0.4\linewidth]{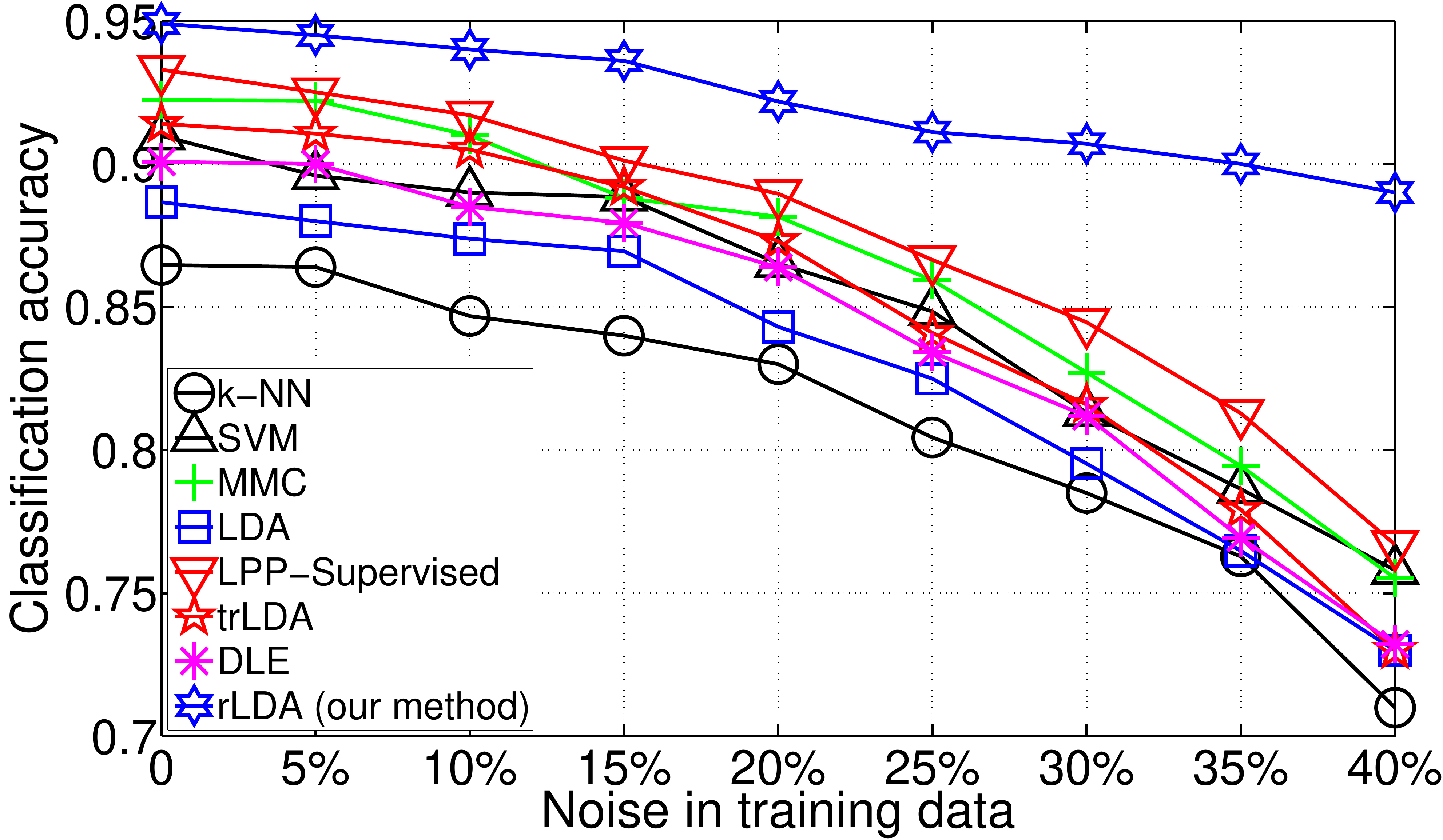}}\qquad
    \subfigure[TDT2 corpus data set.]{\label{fig:acc_var_tdt2}\includegraphics[width=0.4\linewidth]{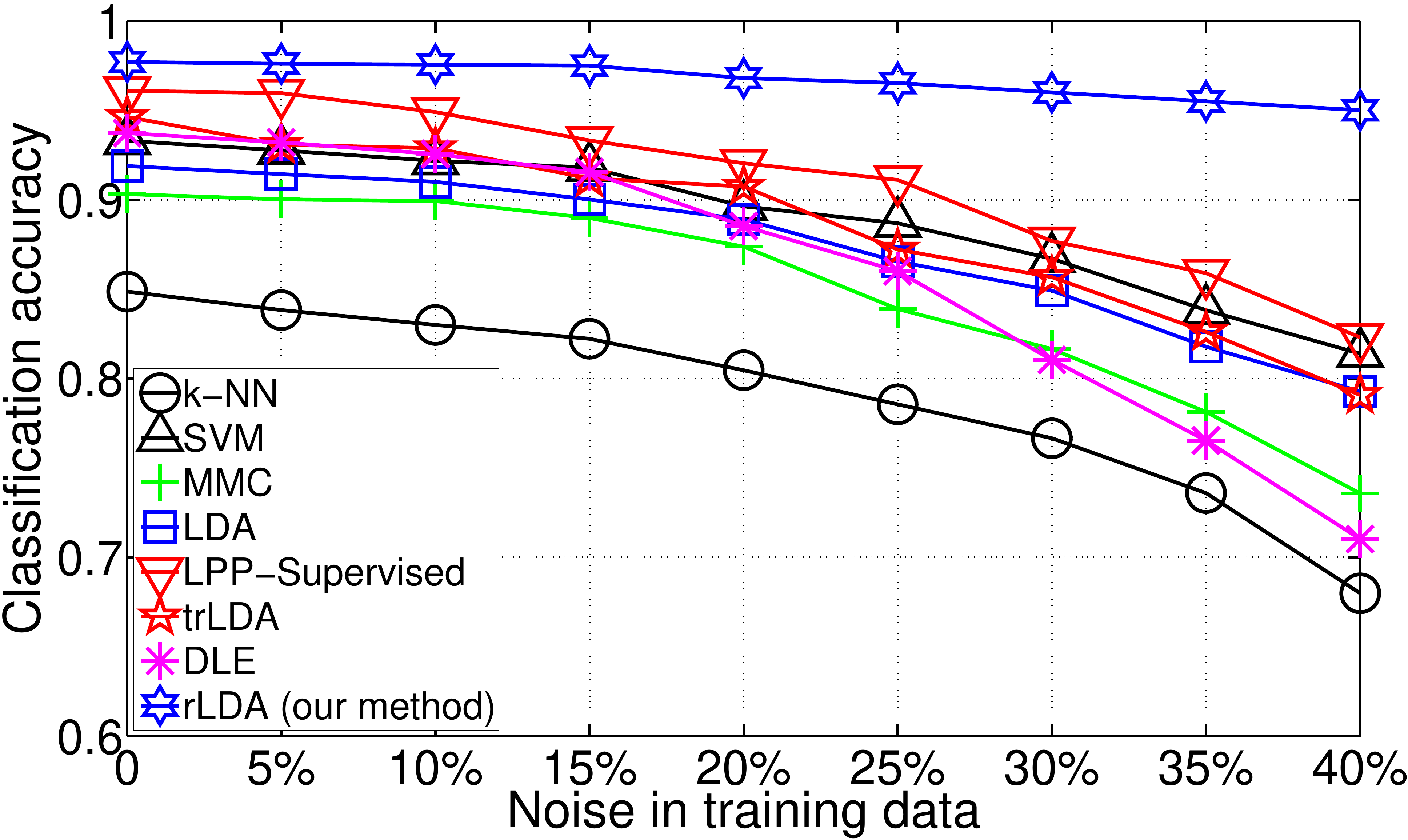}}\qquad
    \subfigure[20Newsgroups data set.]{\label{fig:acc_var_20news}\includegraphics[width=0.4\linewidth]{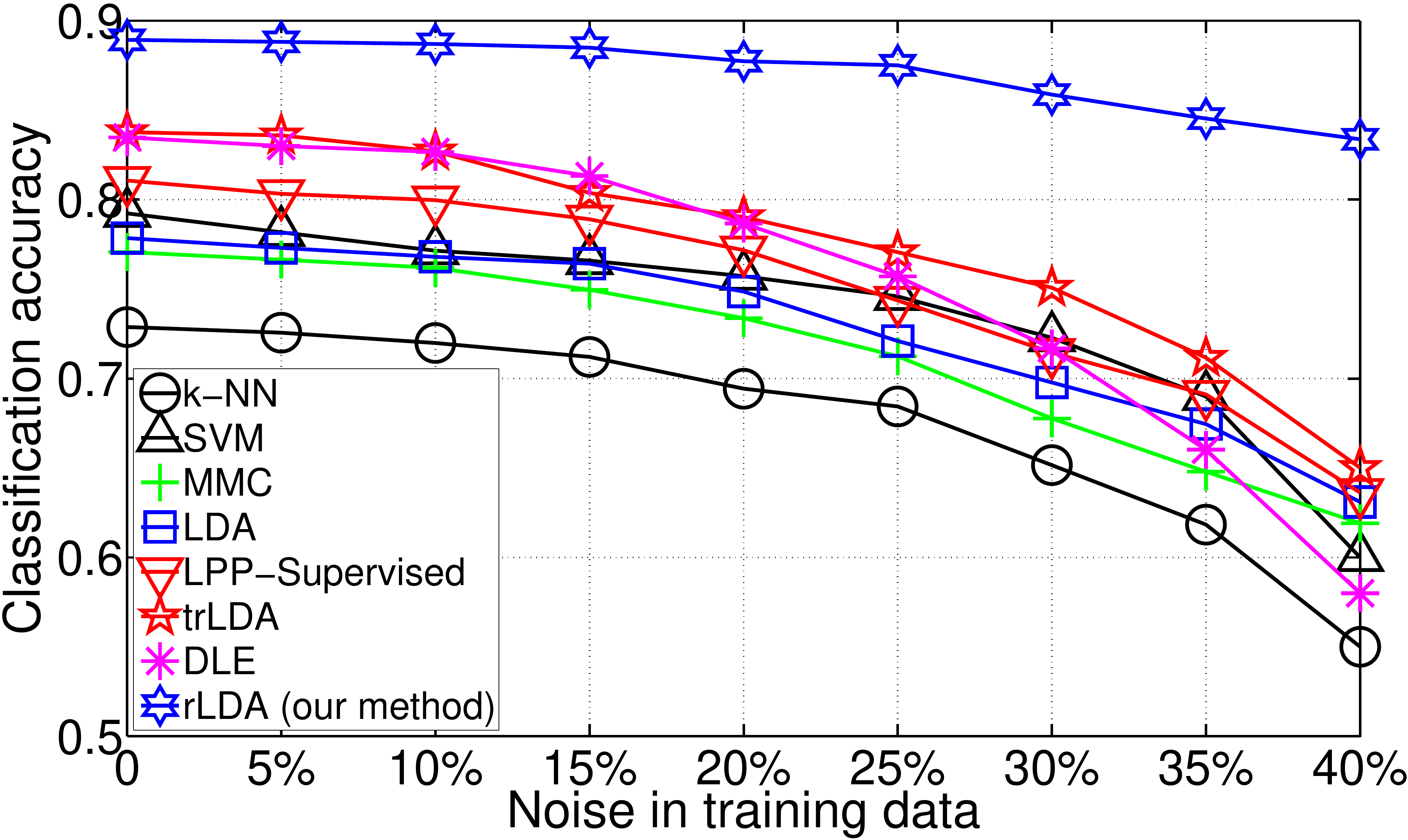}}
  \caption{Classification accuracy of compared methods on all nine data sets when imposed noise in training data varies. The performance of our method is considerably stable against the variations of noise.}
  \label{fig:acc_var}
\end{figure*}

\subsection{Computational Complexity Analysis of Algorithm \ref{alg1}}

The main computational burden of Algorithm \ref{alg1} is from step 4 and step 5. Calculating $A$ costs $O(d^2n)$, and calculating $W_{t+1}$ costs $O(d^3)$. Therefore, the computational complexity of Algorithm \ref{alg1} is $O(d^2pT)$, where $p=max(d,n)$ and $T$ is the iteration number. Note that the computational complexity is the same as that of the trace-ratio LDA method \cite{traceratioTNN09}, thus the proposed algorithm does not increase computational burden while improving robustness of the trace-ratio LDA method. In practice, we find that the proposed algorithm always converges within 5-20 iterations.

\section{Experimental Results}
\label{sec:exp}

In this section, we will evaluate the proposed \mnf method by nine benchmark data sets, and compared it with the state-of-the-art methods.

\subsection{Demonstration Using Synthetic Data}
\label{sec:exp_toy}

First, we use a synthetic data to verify the effectiveness of our rLDA method. The toy data include data sampled from two different Gaussian distributions, and also include a few data that are not sampled from the two distribution and acted as noise data. The projection direction found by LDA and our rLDA are drawn in Fig.~(\ref{toy}). It is known that LDA is optimal if the data of classes come from Gaussian distributions with a single shared covariance. From the result of Fig.~(\ref{toy}), we can see that LDA can not find the optimal direction when there exist noise data in the data, which is consistent with the analysis in the paper that LDA is sensitive to outliers. The result in Fig.~(\ref{toy}) shows that the proposed rLDA can find the optimal direction in this noise case, which indicates that the proposed rLDA is more robust to outliers than traditional LDA.

\subsection{Experimental Results on Real Benchmark Data Sets}
\label{sec:exp_noiseless}

\Heading{Data sets descriptions.}
We evaluate the proposed method on nine widely used benchmark data sets in machine learning and pattern recognition. The data descriptions are summarized in \tblref{tbl:data}. The first three data sets are obtained from UCI machine learning data repository. The three face data sets are AT\&T face data set \cite{ATT:Data}, YALE face data set \cite{GeBeKr01}, and CMU PIE (Face Pose, Illumination, and Expression) face data set \cite{pie}. All face images are resized to $32 \times 32$ following standard computer vision experimental conventions (reducing the misalignment effects). For the document data sets, following previous studies, for Reuters21578 data set, we remove the keywords appearing less than 50 times and end up with 1599 features; for TDT2 corpus data set, we remove the keywords appearing less than 100 times, and end up with 3157 and 4480 features, respectively.

\begin{table}[t]
\centering
  \caption{Data sets used in our experiments.}
  \label{tbl:data}
  \begin{tabularx}{\linewidth}{lZZZ}
    \toprule
    Data set & Number & Dimension & Classes\\
    \midrule
    Vehicle & 946 & 18 & 4\\
    Dermatology & 366 & 34 & 6\\
    Coil-20 & 1440 & 1024 & 20\\
    ORL Face & 400 & 1024 & 40\\
    Yale Face & 165 & 1024 & 15\\
    PIE Face & 3329 & 1024 & 68\\
    Reuters21578 & 8293 & 1599 & 65\\
    TDT2 Corpus & 9394 & 3157 & 30\\
    20Newsgroups & 18744 & 4480 & 20
    \\\bottomrule
  \end{tabularx}
\end{table}

\Heading{Experimental setups.}
We compare the proposed \mna method against the following related supervised methods.
(1) linear discriminant analysis (LDA) \cite{fukunaga1990introduction}, (2) supervised LPP (LPP-{\scriptsize Supervised}) \cite{he2003lpp}, (3) maximum margin criterion (MMC) \cite{li2003mmc}, (4) trace-ratio LDA (trLDA) \cite{traceratioTNN09}, and (5) discriminant Laplacian embedding (DLE) \cite{wang2010discriminant} methods. In addition, as baselines, we also report the classification results by (6) \knna and (7) support vector machine (SVM) in the original feature space.


For LDA and DLE methods, we reduce the data dimensionality to $c-1$, which is usually the maximum rank of $S_b$. Because our approach is supervised, in order for a fair comparison, we implement DLE as a supervised method, \ie, construct the data graph only using labeled data points. Following \cite{traceratioTNN09}, for trLDA and our \mna method, we empirically select the reduced dimensionality to be $3c$. Once the projection matrix is obtained by the dimensionality reduction methods, \knnf method ($k=1$ is used in this work) is used to classify the unlabeled data points in the projected space. In \knna, we use the most widely used Euclidean distance.
We implement SVM by LIBSVM\footnote{\url{http://www.csie.ntu.edu.tw/~cjlin/libsvm/}} package, in which Gaussian kernel (\ie, $\mck\spr{\xcol[i], \xcol[j]} = \exp\spr{-\gamma\normv{\xcol[i] - \xcol[j]}^2}$) is used and the parameters $\gamma$ and $C$ are fine tuned by searching the range of $\sbr{10^{-5}, 10^{-4}, \dots, 10^{4}, 10^{5}}$.


\Heading{Results.}
Because the main advantage of the proposed \mna method is its robustness against noise, we evaluate it on noisy data and study the impact of noise on the classification performance of the proposed method. We conduct standard 5-fold cross validations by all compared methods on every data set.

We randomly pick a certain percentage of the training data in each of the 5 trials, and give them incorrect labels to emulate noise.
We vary the amount of noise imposed in training data in each of the 5 trials of 5-fold cross validation, and show the average classification accuracies of the compared methods on all nine data sets in \fgref{fig:acc_var}.
As can be seen, the performances of the proposed \mna method with respect to noise on all data sets are considerably stable, whereas those of all other methods drop quickly with the increases of noise. These results confirm the robustness of the proposed \mna method when input data are corrupted by outliers.


\section{Conclusions}
We proposed a robust LDA method based on $\ell_{1,2}$-norm ratio minimization which imposes the $\ell_1$-norm between data points and makes the method more robust to outliers. However, the new objective brings the much more challenging optimization problem than the traditional one. We introduced an efficient algorithm to solve the challenging $\ell_{1,2}$-norm ratio minimization problem and provided the rigorous theoretical analysis on the convergence of our algorithm. The new algorithm is easily to be implemented and fast to converge in practice as we have closed form solution in each iteration. We performed extensive experiments on both synthetic data and real data, and all results clearly shown that the proposed method is more effective and robust to outliers than traditional methods.


%
%

%

\section*{Acknowledgment}

This research was partially supported by NSF IIS-1117965, IIS-1302675, IIS-1344152.



%

\bibliographystyle{plainnat}
\bibliography{rlda_pami_arxiv}

%

%
%
%
%




\end{document}